\documentclass{article}

\usepackage{arxiv}

\usepackage[utf8]{inputenc} 
\usepackage[T1]{fontenc}    
\usepackage{hyperref}       
\usepackage{url}            
\usepackage{booktabs}       
\usepackage{amsfonts}       
\usepackage{nicefrac}       
\usepackage{microtype}      
\usepackage{graphicx}
\usepackage{doi}

\usepackage{url}
\usepackage{enumerate}
\usepackage{amsfonts,mathrsfs}
\usepackage{amssymb,amsmath,amsthm,bm}
\usepackage{verbatim}
\usepackage{acronym}
\usepackage{algpseudocode,algorithm,tabularx}
\usepackage{mathtools}
\usepackage{xcolor}
\usepackage{dsfont}
\usepackage{graphicx}
\usepackage{caption}%
\usepackage{subcaption}
\usepackage[title]{appendix}

\newtheorem{theorem}{Theorem}
\newtheorem{lemma}{Lemma}
\newtheorem{assumption}{Assumption}
\newcommand{\tp}{\intercal}
\newcommand{\mc}{\mathcal}
\newcommand{\mb}{\mathbb}

\newcommand{\mcs}{\mathcal{S}}

\newcommand{\mbr}{\mathbb{R}}
\newcommand{\mbz}{\mathbb{Z}}
\newcommand{\mbe}{\mathbb{E}}

\newcommand{\mcp}{\mathcal{P}}

\newcommand{\mdo}{\mathds{1}}
\newcommand{\optcbg}{\texttt{OPT}_{\texttt{CBG}}}
\newcommand{\optbwk}{\texttt{OPT}_{\texttt{BwK}}}

\algdef{SE}[SUBALG]{Indent}{EndIndent}{}{\algorithmicend\ }%
\algtext*{Indent}
\algtext*{EndIndent}

\makeatletter
\newcommand{\multiline}[1]{%
  \begin{tabularx}{\dimexpr\linewidth-\ALG@thistlm}[t]{@{}X@{}}
    #1
  \end{tabularx}
}
\makeatother

\makeatletter
\def\blfootnote{\gdef\@thefnmark{}\@footnotetext}
\makeatother

\title{Online Learning in Budget-Constrained Dynamic Colonel Blotto Games}

\date{} 					

\author{{Vincent Leon}\\
	Department of Industrial and Enterprise Systems Engineering\\
	and Coordinated Science Laboratory\\
	University of Illinois at Urbana-Champaign\\
	Urbana, IL 61801, USA \\
	\texttt{leon18@illinois.edu} \\
	\And
	{S.~Rasoul Etesami}\\
	Department of Industrial and Enterprise Systems Engineering\\
	and Coordinated Science Laboratory\\
	University of Illinois at Urbana-Champaign\\
	Urbana, IL 61801, USA \\
	\texttt{etesami1@illinois.edu} \\
}

\renewcommand{\headeright}{}
\renewcommand{\undertitle}{}
\renewcommand{\shorttitle}{Online Learning in Budget-Constrained Dynamic Colonel Blotto Games}

\hypersetup{
pdftitle={Online Learning in Budget-Constrained Dynamic Colonel Blotto Games},
pdfsubject={cs.LG, cs.GT, cs.MA, eess.SY},
pdfauthor={Vincent Leon, S.~Rasoul Etesami},
pdfkeywords={Colonel Blotto game, dynamic games, game theory, multi-armed bandits, online learning, regret minimization},
}

\begin{document}
\maketitle

\begin{abstract}
In this paper, we study the strategic allocation of limited resources using a Colonel Blotto game (CBG) under a dynamic setting and analyze the problem using an online learning approach. In this model, one of the players is a learner who has limited troops to allocate over a finite time horizon, and the other player is an adversary. In each round, the learner plays a {\color{black} one-shot} Colonel Blotto game with the adversary and strategically determines the allocation of troops among battlefields based on past observations. The adversary chooses its allocation {\color{black} action} randomly from some fixed distribution that is unknown to the learner. The learner's objective is to minimize its regret, which is the difference between the {\color{black} cumulative reward} of the best mixed strategy and the realized {\color{black} cumulative reward} by following a learning algorithm while not violating the budget constraint. The learning in dynamic CBG is analyzed under the framework of combinatorial bandits and bandits with knapsacks. We first convert the budget-constrained dynamic CBG to a path planning problem on a directed graph. We then devise an efficient algorithm that combines a special combinatorial bandit algorithm for path planning problem and a bandits with knapsack algorithm to cope with the budget constraint. The theoretical analysis shows that the learner's regret is bounded by a term sublinear in time horizon and polynomial in other parameters. Finally, we justify our theoretical results by carrying out simulations for various scenarios.
\end{abstract}

\keywords{Colonel Blotto game \and dynamic games \and game theory \and multi-armed bandits \and online learning \and regret minimization}

\section{Introduction}\label{sec:intro}

Colonel Blotto game (CBG) is a classical model of game theory for strategic resource allocation. It was firstly introduced in \cite{borel1921theorie,borel1953theory} and discussed in \cite{frechet1953emile,frechet1953commentary,vonneumann1953communication}. It is a two-player static zero-sum game in which two colonels compete by allocating a limited number of troops over multiple battlefields. Each battlefield has a weight. The colonel who assigns more troops to a battlefield wins that battlefield, and the payoff for a colonel is the sum of the weights of the battlefields that it won. 

Imagine that two companies are going to release new alternative products and compete with each other in the same markets. A common strategy is to advertise the product to some users in the initial stage (e.g., by providing free samples) so that the product becomes popular and possibly dominates the market through ``word-of-mouth'' marketing. When two products are comparable, a user may prefer the product that is better advertised. A company that can get more initial users to adopt its product (especially influential users with high social connectivity) is likely to become the market leader and gain more profit. In this example, the initial users can be regarded as battlefields in the CBG model, and the weights of battlefields can be viewed as a measure of users' social influence. The advertising budgets or the free samples are troops to be strategically allocated. 

As another example, one can consider the United States presidential election. Most states have established a winner-take-all system. Two presidential candidates should properly allocate their campaign resources over different states to win the majority of the votes. Here, one can view the states as the battlefields and the campaign resources as the troops. The number of electoral votes obtained by winning different states can be viewed as the weight of different battlefields. The above scenarios are only two specific examples of CBG, and one can consider many other applications such as cyber-physical security \cite{etesami2019dynamicgames,gupta2014security,min2018apt,guan2020security,labib2015colonel,zhang2022safeguarding}, wireless spectrum allocation \cite{hajimirsadeghi2016internetwork,hajimirsaadeghi2017spectrum}, economics \cite{kovenock2012coalitional}, and political elections \cite{laslier2002distributive,laslier2002how,thomas2018political}.

There are abundant existing works which investigate the Nash equilibria of CBG in discrete and continuous forms under both static and dynamic settings \cite{etesami2019dynamicgames}, for example, \cite{gross1950continuous,roberson2006colonelblotto,ahmadinejad2016fromduels,behnezhad2017faster,behnezhad2022fast,vu2018efficient}. 
However, obtaining Nash equilibrium in either static or dynamic CBG is a challenging task as the players need to know their opponent's budget in order to compute their equilibrium strategies \cite{roberson2006colonelblotto}.\footnote{A static game is a one-shot game where all players act simultaneously. A game is dynamic if players are allowed to act multiple times based on the history of their strategies and observations.} Unfortunately, in many situations, having access to such information is either infeasible or unrealistic. In particular, a player may not necessarily follow the Nash equilibrium dynamics but rather act adversarially by choosing random strategies. Therefore, the other player has to play while learning the adversary's strategies and act strategically. The dynamic CBG that we consider in this paper can be cast as a \emph{multi-armed bandits} (MAB) problem which models a sequential interaction between a \emph{learner} and an \emph{adversary} (or \emph{environment}) \cite{lattimore2020book}. A game is played over a finite time horizon. In each period, the learner selects an \emph{action} (or \emph{arm}) from its action set and receives a reward after interacting with the adversary or the environment. The learner aims to maximize the total reward while learning the game through repeated interactions with the environment and observing the history of outcomes. In the dynamic CBG, an allocation of troops over battlefields represents an arm. The learning player follows an online learning algorithm to minimize its \emph{regret}, which is defined as the difference between the {\color{black} expected cumulative reward obtained by following its best-in-hindsight mixed strategy and the realized cumulative reward obtained by following the learning algorithm (see Section \ref{sec:model} for a formal definition).} Moreover, we consider the dynamic CBG game under a budget constraint. The reason is that in real applications, such as repeated cyberattacks or political elections, the players often have limited security resources or advertising budgets \cite{etesami2019dynamicgames,etesami2021open}.

In this paper, we consider a problem where two players are engaged in a repeated CBG. One player is the learner who aims to maximize its cumulative reward by deciding how to allocate its {\color{black} troops to} the battlefields over a finite time horizon {\color{black} subject to limited budget. The total number of troops allocated by the learner over the entire time horizon cannot exceed its budget. In each stage, the learner has to simultaneously determine the number of troops allocated to that stage and their distribution over the battlefields. The other player is the adversary who chooses its allocation strategies from a fixed distribution unknown to the learner. Therefore, the learner has to learn the adversary's allocation strategy and act accordingly while playing the game. There are two major issues associated with this problem. First, since the learner faces a limited budget, excessive allocation of troops to a single stage may lead to early termination of the dynamic game, which results in large regret. Second, learning the adversary's strategy and estimating the optimal strategy may be computationally expensive because the number of possible allocations is exponentially large. These two issues will be tackled by our algorithm, which makes efficient use of the combinatorial structure of CBG while respecting the budget constraint.
} 

\subsection{Related Work}\label{subsec:related-work}

The Nash equilibrium and equilibrium strategies of CBG have been extensively studied in the past literature. On the one hand, the Nash equilibrium of continuous CBG where resources are infinitely divisible was first studied in \cite{gross1950continuous}. The authors provide a solution to the continuous version when two colonels have the same number of troops. \cite{roberson2006colonelblotto} characterizes the equilibrium strategy of the continuous CBG where the colonels can have any number of troops but all battlefields have the same weights (i.e., homogeneous battlefields). On the other hand, characterizing the optimal strategy for the discrete version of CBG is a difficult problem because the number of allocation strategies grows exponentially in the number of troops and battlefields. \cite{ahmadinejad2016fromduels} proposes an algorithm based on a linear program (LP) with an exponential number of constraints to compute the Nash equilibrium of the discrete CBG. They also make clever use of the ellipsoid method to compute the equilibrium strategies in polynomial time. More recent work \cite{behnezhad2017faster,behnezhad2022fast} provides a polynomial-size LP for finding the equilibrium strategies of discrete CBG. Moreover, the authors of \cite{vu2018efficient} propose a more efficient algorithm to compute an approximate Nash equilibrium. Other variants and generalizations of CBG are also proposed and studied in the past literature, such as generalized CBG \cite{ferdowsi2018generalized}, in which the discontinuous battlefields reward functions are smoothed by an inverse tangent function, and the extended versions of CBG in \cite{hortala2012pure,kovenock2021generalizations}, where the colonels are allowed to have different valuations over battlefields, and hence the game is generalized to a general-sum game.

All the above work focuses on static CBG, which is a single-shot game. In a dynamic CBG with a budget constraint, the players act strategically based on the observed history repeatedly over a finite time horizon, and the allocation strategies across time periods are essentially correlated and coupled.
The dynamic CBG can be viewed as a two-level problem. On the upper level, the learner is faced with an optimization problem for the distribution of budget over the time horizon. On the lower level, in each round, the learner must play a one-shot CBG to strategically distribute the designated budget for that round over the battlefields. 

Our work is also related to \cite{hajimirsadeghi2016internetwork,hajimirsaadeghi2017spectrum,vu2019combinatorial}. \cite{hajimirsadeghi2016internetwork} studies the spectrum allocation problem using a CBG model and proposes a learning algorithm based on fictitious play to compute the Nash equilibrium of static CBG numerically. However, the authors do not consider the playing-while-learning scenario, and the purpose of the learning algorithm is numerical convergence rather than strategic action. The essence of \cite{hajimirsadeghi2016internetwork} is still a static game.
In \cite{hajimirsaadeghi2017spectrum}, the authors study a dynamic CBG model for spectrum resource allocation problem where two network service providers offer limited spectrum resources to multiple users over a finite time horizon. The problem is formulated as a discrete-time dynamic game, and the saddle-point equilibrium strategies are found under some technical assumptions. However, \cite{hajimirsaadeghi2017spectrum} does not involve any learning dynamics and only considers the upper-level model ignoring the lower-level strategy. Authors in \cite{vu2019combinatorial} propose an online learning algorithm for repeated CBG where the learning player learns the history of actions and rewards and acts strategically. Different from our model, they do not impose a total budget constraint; instead, the per-round budget of the learner is controlled. \cite{vu2019combinatorial} ignores the upper-level budget allocation because the unused budget in one stage will not be accumulated to the subsequent stages. Therefore, we will use new ideas to incorporate the budget constraint into our game dynamics.

The problem we consider in this paper also falls under the framework of the MAB problem. More specifically, it lies in the intersection of \emph{combinatorial bandits} (CB) and \emph{bandits with knapsacks} (BwK) because of the combinatorial structure of CBG and the constraint on the total budget. The CB was introduced in \cite{cesabianchi2012}. They propose a so-called \textsc{ComBand} multiplicative weight algorithm for CB under the bandit setting and show a regret bound of  $O(\sqrt{T \ln S})$ where $T$ is the time horizon, and $S$ is the number of actions that may be exponentially large. \cite{combes2015combinatorial} improves the time complexity of the algorithm while maintaining the same regret bound; it also provides an improved performance guarantee in some special cases. \cite{vu2019combinatorial} extends the algorithm \textsc{ComBand} to repeated CBG without budget constraint and obtains the same performance guarantee while maintaining the algorithm's efficiency. They use a clever way of transforming the CBG model into a directed graph and reducing the size of the problem to a polynomial size. 
BwK is an MAB model with budget constraint, first introduced by \cite{badanidiyuru2013bwkabstract} and subsequently studied in \cite{agrawal2014bandits,agrawal2016linear,agrawal2016efficient,immorlica2020lagrangebwk,li2021symmetry}. \cite{badanidiyuru2013bwkabstract} introduces the BwK model and proposes two LP-based algorithms for stochastic BwK that solve the online problem optimally up to a logarithmic factor. \cite{agrawal2014bandits} studies the BwK problem where a more general notion of reward is used, and the hard budget limit is softened. \cite{agrawal2016linear,agrawal2016efficient} study a special case of BwK named contextual BwK in which the learner receives an additional context in each round. \cite{immorlica2020lagrangebwk} proposes an LP-based algorithm named \texttt{LagrangeBwK} for stochastic BwK and obtains a near-optimal performance guarantee on regret. The algorithm \texttt{LagrangeBwK} is based on the Lagrangian relaxation and the dual program of a linear program. It can be viewed as a ``black box'' that uses two problem-specific online learning algorithms as subroutines to compete with each other and converge to the Nash equilibrium of the game induced from the Lagrangian function of the LP (see Section \ref{subsec:prelim-LagrangeBwK} for more details). The recent work \cite{li2021symmetry} also proposes a primal-dual based algorithm for stochastic BwK, which achieves a problem-dependent logarithmic regret bound.

\subsection{Contributions and Organization} 

In this paper, we {\color{black} study the problem of dynamic CBG where the learner faces a budget constraint} and develop an online learning algorithm for the {\color{black} learner in the} dynamic CBG. {\color{black} The algorithm} achieves a regret bound sublinear in time horizon $T$ and polynomial in other parameters. 
{\color{black} To deal with the issue of limited total budget,} we use a tailored version of {\color{black} an algorithm named} \texttt{LagrangeBwK}, {\color{black} which is designed for MAB with knapsack constraints.}
{\color{black} To deal with the time complexity issue due to exponentially large allocation action set,} we transform the CBG model to a path planning problem on a directed graph {\color{black} and apply an efficient CB algorithm named \textsc{Edge} from \cite{vu2019combinatorial}}, which is originally intended for the repeated CBG without total budget constraint. {\color{black} This technique effectively utilizes the combinatorial structure of CBG and achieves polynomial time complexity.}
To extend that result to dynamic CBG with total budget constraint, we provide a revised directed graph for path planning problem and use the algorithm \textsc{Edge} as a subroutine of {\color{black} the algorithm \texttt{LagrangeBwK} to obtain a provable regret bound.}

The paper is organized as follows. In Section \ref{sec:model}, we formally introduce the dynamic CBG model {\color{black} and the learning objectives}. In Section \ref{sec:prelim}, we provide preliminary results for two essential algorithms, namely \texttt{LagrangeBwK} and \textsc{Edge}. In Section \ref{sec:main-results}, we present our devised algorithm \texttt{LagrangeBwK-Edge} for the dynamic CBG together with the regret analysis. Simulation results are presented in Section \ref{sec:simulation}. We conclude the paper in Section \ref{sec:conclusion}. 

\noindent
{\bf Notation:} Throughout the paper, we use bold fonts to represent vectors (e.g., $\bm{x}$) and subscript indices to denote their components (e.g., $x_i$). We use $\bm{x}^{\tp}$ to denote the transpose of vector $\bm{x}$. For any positive integer $k$, we let $[k] \triangleq \{1, \cdots, k\}$. {\color{black}Finally, we use $\mdo\{\cdot\}$ to represent the indicator function, i.e., $\mdo\{A\} = 1$ if event $A$ happens, and $\mdo\{A\} = 0$ otherwise.}


\section{Problem Formulation} \label{sec:model}


We consider a {\color{black} dynamic game} with two players: one learner and one adversary. There are $T\in \mathbb{Z}_+$ rounds, where at each round $t \in [T]$, the two players play a static one-shot CBG {\color{black} described as follows}. The static CBG consists of $n$ battlefields, and each battlefield $i\in \left[n\right]$ has fixed weight $b_i > 0$. The weights of all battlefields are not necessarily identical, but the sum of weights is normalized to $1$, i.e., \(\sum_{i=1}^n b_i = 1\). Also, the weights of battlefields $\{b_i\}_{i=1}^n$ are not necessarily known to the learner. The players choose {\color{black} an allocation action} by determining the {\color{black} number of troops} allocated to each battlefield without knowing the opponent's {\color{black} action}. Without loss of generality, we assume that the {\color{black} number of troops} distributed to each battlefield is integer-valued. A player wins a battlefield if it allocates more {\color{black} troops} to that battlefield than its opponent and receives a reward equal to that battlefield's weight. If there is a tie, the weight is shared equally between two players. Finally, the total reward of a player equals the sum of its rewards {\color{black} over} all battlefields. {\color{black} At the end of each round,} the learner observes the total reward that it receives from all battlefields, but it does not know which battlefields that it has won or lost.

{\color{black}
In \emph{dynamic} CBG, the learner and the adversary play the one-shot CBG repeatedly. The learner has a fixed total budget $B \in \mbz_+$ for the entire time horizon. The total number of troops to be allocated by the learner through $T$ rounds cannot exceed $B$.} The dynamic CBG terminates when the time horizon is reached or the {\color{black} learner's} budget is exhausted. 
{\color{black} No budget constraint is imposed on the adversary. However, it is assumed that in each round, the adversary chooses an allocation action according to some fixed distribution $\mcp_{\text{adv}}$ unknown to the learner. Such a distribution is also referred to as \emph{mixed strategy} as the adversary randomizes over its action set and chooses an allocation action from the action set according to the distribution $\mcp_{\text{adv}}$. Furthermore, we make the following assumption which is} realistic from a practical point of view as the learner is not willing to {\color{black} consume} too much budget in one {\color{black} single} stage to sacrifice the long-term reward.


\begin{assumption}\label{assump:c}
We assume that the maximum budget {\color{black} consumed} by the learner in each stage is at most $m = cB/T$, where {\color{black} $c \in \mbr_+$} is some constant. That is, the learner is not willing to allocate $c$ times more than the per-round average budget to any stage. {\color{black} Without loss of generality, we may assume $m \in \mbz_+$.}
\end{assumption}

{\color{black}
In fact, if the distribution of the adversary's strategy $\mcp_{\text{adv}}$ were known to the learner, the learner could determine the optimal mixed strategy by solving a linear program (LP). Firstly, let $\bm{u}_t = (u_{t,1}, \cdots, u_{t,n}) \in \mbz_+^n$ and $\bm{v}_t = (v_{t,1}, \cdots, v_{t,n}) \in \mbz_+^n$ represent the learner's and the adversary's allocation actions in round $t \in [T]$ respectively, where $u_{t,i}$ and $v_{t,i}$ denote the number of troops allocated to battlefield $i \in [n]$ by the learner and the adversary, respectively.
The learner receives a reward 
\begin{equation}\label{eq:reward-CBG}
    r_t(\bm{u}_t) = \sum_{i=1}^n b_i \left( \mdo\{u_{t,i} > v_{t,i}\} + \frac{\mdo\{u_{t,i} = v_{t,i}\}}{2} \right),
\end{equation}
where $\mdo\{\cdot\}$ is the indicator function. In fact, the reward function in \eqref{eq:reward-CBG} should be of the form  $r(\bm{u}_t,\bm{v}_t)$ as it is a function of both $\bm{u}_t$ and $\bm{v}_t$. However, for simplicity of notation, in the remainder of this work we suppress the dependency on the adversary's action $\bm{v}_t$ by simply writing $r_t(\bm{u}_t)=r(\bm{u}_t,\bm{v}_t)$.

Since the adversary's allocation action $\bm{v}_t$ is random and $\bm{v}_t \sim \mcp_{\text{adv}}$, $r_t(\bm{u}_t)$ is also a random variable and follows some fixed distribution. Moreover, the learner's set of allocation actions is represented as 
\begin{equation}\label{eq:action-set}
    \mc{S} = \left\{\bm{u} \in \mb{Z}_+^n: \sum_{i=1}^n u_i \leq m \right\},
\end{equation}
where $m = cB/T$ is defined in Assumption \ref{assump:c} and represents the maximum number of troops that the learner is willing to allocate in one single stage. Note that $\mcs$ is a finite discrete set and contains exponentially many elements.
The learner's LP can be formulated as follows. Associated with each $\bm{u} \in \mcs$, there is a variable $y(\bm{u})$ which represents the probability that the learner chooses allocation action $\bm{u}$ in each round. The learner aims to maximize the expected cumulative reward, where the expectation is taken with respect to the learner's action distribution $(y(\bm{u}))_{\bm{u} \in \mcs}$ and the randomness of the adversary $\mcp_{\text{adv}}$. Thus, if we define $r(\bm{u}) = \mbe[r_t(\bm{u}) \vert \bm{v}_t \sim \mcp_{\text{adv}}]$, the objective function of the LP is 
\begin{align}\nonumber
\sum_{t=1}^T \mbe \left[ r_t(\bm{u}_t) \vert \bm{v}_t \sim \mcp_{\text{adv}}, \bm{u}_t \sim y(\cdot) \right]= T\cdot \sum_{\bm{u} \in \mcs} r(\bm{u}) y(\bm{u}).
\end{align}
Furthermore, the expected cumulative consumption of troops cannot exceed the total budget $B$. 
Therefore, the LP is formulated as follows: 
\begin{align}
\max \quad & 
T \cdot \sum_{\bm{u} \in \mcs} r(\bm{u}) y(\bm{u}), \label{eq:cbg-LP-obj} \\
\text{s.t.} \quad & T \cdot \sum_{\bm{u} \in \mcs} \left(\sum_{i=1}^n u_i\right) y(\bm{u}) \leq B, \label{eq:cbg-LP-budget} \\
& \sum_{\bm{u} \in \mcs} y(\bm{u}) = 1, \qquad y(\bm{u}) \geq 0 \quad \forall \bm{u} \in \mcs. \label{eq:cbg-LP-distr}
\end{align}
Note that in the above LP, the only decision variables are $\{y(\bm{u})\}_{\bm{u}\in\mcs}$ since $r(\bm{u})$ can be determined given $\mcp_{\text{adv}}$, and $\sum_{i=1}^n u_i$ can be easily computed for each $\bm{u} \in \mcs$.

Unfortunately, since the adversary's strategy distribution $\mcp_{\text{adv}}$ is unknown to the learner, solving the above LP becomes impossible because $r(\bm{u})$ in Eq. \eqref{eq:cbg-LP-obj} cannot be evaluated.
For that reason, we will use an online-learning approach for multi-armed bandits where the learner learns the adversary's strategy distribution $\mcp_{\text{adv}}$ while interacting with it.
In the dynamic CBG, starting from $t=1$, the learner selects an allocation action $\bm{u}_t \in \mcs$ in each round $t$ until some round $\tau \in [T+1]$, where $\tau$ is the first round when either the time horizon is reached (i.e., $\tau = T+1$), or the consumption in that round exceeds learner's remaining budget for the first time (i.e., the budget is exhausted, $\sum_{i=1}^n u_{\tau,i} > B - \sum_{t=1}^{\tau-1} \sum_{i=1}^n u_{t,i}$). Such a $\tau$ is defined as the \emph{stopping time} of the dynamic game as the game terminates in round $\tau$. Note that $\tau$ is a random variable. In each round $t \in [\tau-1]$, the learner selects an allocation action $\bm{u}_t \in \mcs$ based on the history of its actions $\{\bm{u}_s\}_{s=1}^{t-1}$ and its realized rewards $\{r_s(\bm{u}_s)\}_{s=1}^{t-1}$ using some learning algorithm and receives a reward $r_t(\bm{u}_t)$. The realized cumulative reward of the learner over the time horizon is $\sum_{t=1}^{T} r_t(\bm{u}_t)=\sum_{t=1}^{\tau-1} r_t(\bm{u}_t)$. 
The performance of the learning algorithm is compared to the optimal solution of the LP represented by Eq. \eqref{eq:cbg-LP-obj}-\eqref{eq:cbg-LP-distr} and is measured by regret. More precisely, denote by $\optcbg$ the optimal value of the LP represented by Eq. \eqref{eq:cbg-LP-obj}-\eqref{eq:cbg-LP-distr}. It is also called the \emph{best-in-hindsight} expected cumulative reward of the learner because it is the optimal achievable expected total reward over the time horizon if the learner had known the adversary's strategy distribution $\mcp_{\text{adv}}$. The optimal solution to the LP is defined accordingly as the best-in-hindsight mixed strategy or optimal mixed strategy. 
The \emph{realized regret} of the learner is defined by $R(T) \triangleq \optcbg - \sum_{t=1}^T r_t(\bm{u}_t)$. 
The performance of the learning algorithm is measured by $R(T)$.
Our main objective is to devise a polynomial-time learning algorithm that achieves a realized regret $R(T)$ that is sublinear in $T$ and scales polynomially in other problem parameters with high probability.}

\section{Preliminaries} \label{sec:prelim}

In this section, we provide a brief overview of the two key algorithms that we will use later to establish our main {\color{black} algorithm and the} theoretical results. 

\subsection{Lagrangian Game and  \texttt{LagrangeBwK} Algorithm} \label{subsec:prelim-LagrangeBwK}

We start with \texttt{LagrangeBwK}, which is an algorithm for the stochastic BwK initially introduced in \cite{immorlica2020lagrangebwk}. It is based on the linear relaxation of the stochastic BwK problem and the corresponding Lagrange function. More precisely, consider an MAB problem subject to knapsack constraints. {\color{black} In the MAB problem, the learner makes sequential decisions over the time horizon $T$. The learner has a finite set of actions (arms) denoted by $\mcs$ and $d$ types of resources. For each resource $i \in [d]$, the learner has a fixed budget $B_i$. Pulling an arm $a \in \mcs$ consumes $w_i(a)$ unit of resource $i$ in expectation and returns an expected reward $r(a)$. After proper scaling, we assume that the learner has the same budget for each resource $i \in [d]$, i.e., $B_i = B_0 \leq T \: \forall i \in [d]$. Without loss of generality, we assume that time is also a type of resource with budget $B_0$, and pulling any arm $a\in \mcs$ deterministically consumes $B_0/T$ unit of time resource. The learner aims to maximize its expected cumulative reward while respecting the budget constraint of each resource. Define $(y(a))_{a\in\mcs}$ as the distribution over the action space. Therefore, this leads to the following LP for the learner:
\begin{align}
    \max \quad & \sum_{a \in \mcs} r(a) y(a), \label{eq:bwk-LP-obj} \\
    \text{s.t.} \quad & \sum_{a \in \mcs} w_i(a) y(a) \leq \frac{B_0}{T} \quad \forall i \in [d], \label{eq:bwk-LP-budget} \\
    & \sum_{a \in \mcs} y(a) = 1, \qquad 
    y(a) \geq 0 \quad \forall a \in \mcs.\label{eq:bwk-LP-distr}
\end{align}
By comparing the LP defined earlier for the dynamic CBG problem by Eq. \eqref{eq:cbg-LP-obj}-\eqref{eq:cbg-LP-distr} with this LP represented by Eq. \eqref{eq:bwk-LP-obj}-\eqref{eq:bwk-LP-distr}, one can see that the former is a special case of the latter LP where there is only one type of resource (i.e., the troops) apart from the time resource, and the objective function is scaled by $T$. The action set $\mcs$ is exactly the set of allocation actions in Eq. \eqref{eq:action-set}.
}

The Lagrangian function of the LP obtained by relaxing the knapsack constraints represented by Eq. \eqref{eq:bwk-LP-budget} using dual variables $\bm{\lambda} = (\lambda_1, \cdots, \lambda_d)$ is given by 
\begin{equation} \label{eq:LP-Lagrange}
    \mc{L}\left(\bm{y}, \bm{\lambda}\right) = \sum_{a \in \mcs} r(a) y(a) + \sum_{i=1}^d \lambda_i \left(1 - \frac{T}{B_0} \sum_{a \in \mcs}  w_i(a) y(a) \right).
\end{equation}
The Lagrangian function in Eq. \eqref{eq:LP-Lagrange} induces a two-player zero-sum game which is called \emph{Lagrangian game}. The Lagrangian game is described as follows. There are two players in the game: a primal player and a dual player.
The primal player selects an arm \(a \in \mcs\), and the dual player selects a resource \(i \in [d]\). The payoff for any pair of primal-dual actions is given by
\begin{equation} \label{eq:Lagrange-game-expected-payoff}
    \mc{L}(a, i) = r(a) + 1 - \frac{T}{B_0} w_i(a), \quad \forall a\in \mcs, i\in [d]. 
\end{equation}
In particular, $\mc{L}(a,i)$ is the reward for the primal player to maximize and the cost for the dual player to minimize. 
{\color{black} Moreover, the mixed strategies of the primal and dual players are the primal and dual variables $(\bm{y}, \bm{\lambda})$ in the Lagrangian function in Eq. \eqref{eq:LP-Lagrange}.} The Nash equilibrium of the Lagrangian game, denoted by \((\bm{y}^*, \bm{\lambda}^*)\), yields the minimax value $\mc{L}^*$ of the Lagrangian function in Eq. \eqref{eq:LP-Lagrange} which further equals the optimal value of the LP in Eq. \eqref{eq:bwk-LP-obj}-\eqref{eq:bwk-LP-distr} denoted by $\optbwk$. $\bm{y}^*$ also forms an optimal solution to the LP in Eq. \eqref{eq:bwk-LP-obj}-\eqref{eq:bwk-LP-distr} and corresponds to the best-in-hindsight mixed strategy of the learner \cite{badanidiyuru2013bwkabstract,immorlica2019lagrangebwk}.

{\color{black}
Based on the Lagrangian game, the authors in \cite{immorlica2019lagrangebwk} have devised a learning algorithm called \texttt{LagrangeBwK} where the learner creates a repeated Lagrangian game on top of the interactions with the adversary. The Lagrangian game consists of primal and dual players who play against each other. The learner deploys a primal algorithm ($\texttt{ALG}_1$) for the primal player who aims to maximize the payoff of the Lagrangian game and deploys a dual algorithm ($\texttt{ALG}_2$) for the dual player who aims to minimize the payoff. $\texttt{ALG}_1$ and $\texttt{ALG}_2$ are online learning algorithms for adversarial MAB without budget constraints, for example, EXP3.P \cite{auer2002nonstochastic}, Hedge \cite{freund1997hedge} (see Algorithm \ref{alg:hedge} in Appendix \ref{ap:algr}), and \textsc{ComBand} \cite{cesabianchi2012} algorithms.
We note that the learner operates both primal and dual algorithms $\texttt{ALG}_1$ and $\texttt{ALG}_2$ in the Lagrangian game. The adversary's actions merely change the payoff structure of the Lagrangian games at different rounds denoted by $\mc{L}_t(a_t, i_t)$\footnote{{\color{black}Here, $\mc{L}_t(a_t, i_t)=r_t(a_t)+1+\frac{T}{B_0}w_{t,i_t}(a_t)$ is the the Lagrangian function at round $t$, where the subscript $t$ again suppresses the dependency of the parameters on the adversary's action.}}. The learner's budget constraint is taken into consideration through the learner's payoff structure in the Lagrangian game. In each round $t \in [T]$, the primal and dual players play the Lagrangian game and generate an action pair $(a_t, i_t) \in \mcs \times [d]$. The learner gathers this information $(a_t, i_t)$ from the Lagrangian game and chooses the action $a_t$ against the adversary. Then, the learner receives a reward $r_t(a_t)$ and observes the consumption of resources $w_{t,i}(a_t) \: \forall i \in [d]$. Next, the learner computes the realized payoff of the Lagrangian game $\mc{L}_t(a_t, i)\: \forall i\in [d]$ using Eq. \eqref{eq:Lagrange-game-expected-payoff} and feeds this information to the Lagrangian game for the primal and dual players to update their strategies for the next round. The fictitious primal and dual players play and learn against each other in the repeated Lagrangian game, which leads to gradual convergence to the Nash equilibrium of the Lagrangian game, which in turn corresponds to the optimal mixed strategy of the learner against the adversary.}


The pseudocode of algorithm \texttt{LagrangeBwK} is summarized in Algorithm \ref{alg:LagrangeBwK}.
It has been shown that the average play of the primal and dual algorithms \((\bar{\bm{y}}, \bar{\bm{\lambda}})\) by following \texttt{LagrangeBwK} forms an approximate Nash equilibrium of the Lagrangian game, and the average payoff $\bar{\mc{L}}$ converges to {\color{black} $\optbwk$. The regret of algorithm \texttt{LagrangeBwK} is defined similarly as $T \cdot \optbwk - \sum_{t=1}^T r_t(a_t)$.}
An upper bound for the regret of algorithm \texttt{LagrangeBwK} is given in the following lemma.

\begin{lemma}\cite[Theorem 4.4]{immorlica2020lagrangebwk} \label{lemma:LagrangeBwK-regret}
Fix an arbitrary failure probability $\delta \in (0,1)$, and let $R_{1,\delta}(T)$ and $R_{2,\delta}(T)$ denote the high-probability regret bounds of $\texttt{ALG}_1$ and $\texttt{ALG}_2$, respectively. With probability at least \(1 - O(\delta T)\), the regret of algorithm \texttt{LagrangeBwK} is at most \(O\left(\frac{T}{B_0}\right)\left(R_{1,\delta/T}(T) + R_{2,\delta/T}(T) + \sqrt{T \log (dT/\delta)}\right)\). 
\end{lemma}


\begin{algorithm}[t]
\caption{\texttt{LagrangeBwK} Algorithm} \label{alg:LagrangeBwK}
\textbf{Input:} $B_0$, $T$, primal algorithm $\texttt{ALG}_1$, dual algorithm $\texttt{ALG}_2$. 
\begin{algorithmic}[1]
    \For{round $t = 1, \cdots, T$}
    \State {\color{black} Learner's} $\texttt{ALG}_1$ chooses arm $a_t \in \mcs$. 
    \State {\color{black} Learner's} $\texttt{ALG}_2$ chooses resource $i_t \in [d]$. 
    \State Observe $r_t(a_t)$ and $w_{t,i}(a_t), \forall i \in [d]$. 
    \State $\mc{L}_t(a_t, i_t)$ is {\color{black} computed} and returned to $\texttt{ALG}_1$ as reward. 
    \State $\mc{L}_t(a_t, i)$ is {\color{black} computed} and returned to $\texttt{ALG}_2$ as cost for each $i \in [d]$. 
    \EndFor 
  \end{algorithmic}
\end{algorithm}

The algorithm \texttt{LagrangeBwK} can serve as the main algorithm for {\color{black} the learner in the dynamic CBG with a budget constraint}. Unfortunately, a direct application of \texttt{LagrangeBwK} for dynamic CBG results in a large regret due to the exponential number of allocation actions. We shall circumvent that issue by leveraging the combinatorial structure of CBG and using another algorithm \textsc{Edge} in conjunction with \texttt{LagrangeBwK}. Here, the algorithm \textsc{Edge} is an efficient algorithm for regret minimization in a repeated CBG without any total budget constraint \cite{vu2019combinatorial}.

\subsection{Path Planning Graph and Algorithm \textsc{Edge}}
 
The algorithm \textsc{Edge}, an efficient variant of the combinatorial MAB algorithm, has been used in \cite{vu2019combinatorial} for the CBG without budget constraint. The main idea of the algorithm \textsc{Edge} is to transform the CBG to a path planning problem on a directed layered graph $G_{m,n}=(\mathcal{V},\mathcal{E})$, with a one-to-one bijection between the action set of the CBG and the set of paths of the layered graph \cite{vu2019combinatorial}. In their problem setting, the one-shot CBG with $n$ battlefields and $m$ troops is repeated for $T$ times. The learner has to allocate all $m$ troops in each stage because any unallocated troops will not be accumulated for the next round.

\begin{figure}[t]
\centering
\includegraphics[width=.6\textwidth]{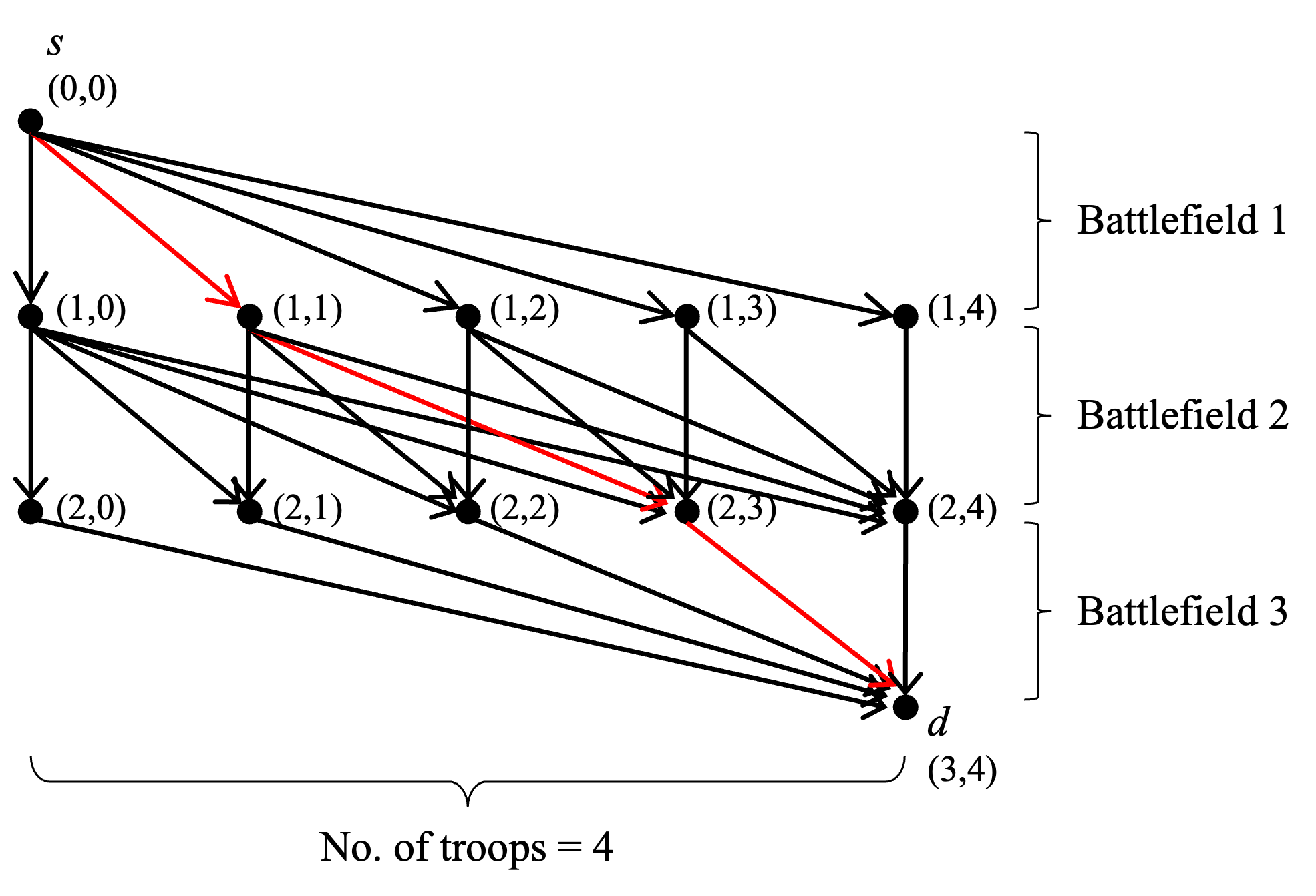}
\caption{An example of layered graph $G_{4,3}$ for CBG with $m=4$ and $n=3$. The red path represents the allocation action $(1,2,1)$ with total consumption 4.}
\label{fig:sample-path-for-edge}
\end{figure}

For a one-shot CBG, one can construct a graph containing $n+1$ layers (layer $0, 1, \cdots, n$), and each layer $i \in [n-1]$ has $m+1$ nodes denoted by $(i,0), (i,1), \cdots, (i,m)$. Layer $0$ has one node $(0,0)$ (namely, \emph{source node} $s$), and layer $n$ has one node $(n,m)$ (namely, \emph{destination node} $d$). There is a directed edge from $(i-1, j)$ to $(i, j')$ if and only if $j' \geq j$, indicating that the quantity of troops assigned to battlefield $i$ equals $j' - j$. 
We refer to Figure \ref{fig:sample-path-for-edge} for an example of the layered graph corresponding to a repeated CBG with $n=3$ battlefields and $m=4$ troops. 
By construction, it is easy to see that any allocation action for the one-shot CBG corresponds to a unique $s,d$-path in the layered graph, and vice versa. 

Instead of sampling paths directly, which corresponds to sampling strategies in the repeated CBG, the algorithm \textsc{Edge} samples the edges that can form an $s,d$-path in each round $t\in[T]$. Doing so, one can design polynomial-time algorithms for sampling edges rather than dealing with exponentially many paths. 
This can be done by assigning weights to the edges and updating or maintaining the edge weights.
{\color{black} Now, if we denote the set of all $s,d$-paths as $\mcs$ and let $E = \vert\mc{E}\vert$, a path $\bm{u} \in \mcs$ is represented by an $E$-dimensional binary vector.}
We refer to Algorithm \ref{alg:Edge} for a detailed description of the algorithm \textsc{Edge}. One can upper-bound the expected regret of the algorithm \textsc{Edge} for the repeated CBG as follows:

\begin{lemma}\cite[Proposition 4.1]{vu2019combinatorial} \label{lemma:Edge-regret}
The algorithm \textsc{Edge} yields an expected regret at most 
\(O\left(\sqrt{\left(\frac{n}{E \lambda^*} + 1\right) ET \ln S}\right)\), where {\color{black} $S = \vert\mcs\vert$}, and $\lambda^*$ is the smallest nonzero eigenvalue of the co-occurrence matrix \(M(\mu)\) of the exploration distribution $\mu$. {\color{black} Here, $\mu$ is a distribution over $\mcs$ and is an input into the algorithm, and $M(\mu) \triangleq \mbe_{\bm{u}\sim\mu}[\bm{u}\bm{u}^{\tp}]$.} Moreover, the running time of the algorithm is at most $O(n^2 m^4 T)$.
\end{lemma}

\begin{algorithm}[t]
\caption{\textsc{Edge($\mu$)} Algorithm} \label{alg:Edge} 
\textbf{Input:} $m$, $n$, $T$, $\gamma \in [0,1]$, $\eta > 0$, distribution $\mu$ {\color{black} over the set of $s,d$-paths}.
\begin{algorithmic}[1]
    \State \textbf{Initialization:} Edge weights $w_e^1 = 1, \forall e \in \mathcal{E}$.
    \For{round $t = 1, \cdots, T$} 
    \State \multiline{The adversary selects a cost vector $\bm{l}_t$ unobserved by the learner.} 
    \State \underline{\textit{Path sampling}}
    \Indent
    \State With probability $\gamma$, sample a path $\bm{u}_t$ from distribution $\mu$. 
    \State \multiline{Otherwise, sample a path $\bm{u}_t$ by edges' weights according to the Weight-Pushing Algorithm (see Algorithm \ref{alg:wp} in Appendix \ref{ap:algr}).}
    \EndIndent 
    \State Observe the reward $r_t(\bm{u}_t)=\bm{l}_t^{\tp} \bm{u}_t$. 
    \State \underline{\textit{Weight updating}}
    \Indent
    \State \multiline{Compute the co-occurrence matrix $C_t$ using the Co-occurrence Matrix Computation Algorithm (see Algorithm \ref{alg:matcomp} in Appendix \ref{ap:algr}).}
    \State Estimate the cost vector $\hat{\bm{l}}_t = r_t(\bm{u}_t)C_t^{-1}\bm{u}_t$. 
    \State \multiline{Update edges' weights \\ $w_e^{t+1} = w_e^t \cdot \exp (\eta \hat{l}_{t,e})\ \forall e \in \mathcal{E}$, where \(\hat{l}_{t,e}\) is the entry of \(\hat{\bm{l}}_t\) corresponding to \(e\).}
    \EndIndent 
    \EndFor 
\end{algorithmic}
\end{algorithm}




\section{Main Results} \label{sec:main-results}

In this section, we first introduce a modified layered graph that allows us to incorporate the budget constraint into the graph's structure and transform the dynamic CBG into a path planning problem. Then, we present an extended analysis of algorithm \textsc{Edge} in which the realized regret is bounded with high probability. After that, we formalize the algorithm \texttt{LagrangeBwK-Edge} using \textsc{Edge} as a subroutine for \texttt{LagrangeBwK} tailored to the budget-constrained dynamic CBG model. 

\subsection{Path Planning for Dynamic CBG with Budget Constraint} \label{subsec:path-planning}

Inspired by \cite{vu2019combinatorial}, we transfer the dynamic CBG to a path planning problem. There are several differences between their CBG model and ours. First, in their model, there is no constraint on the budget for the entire time horizon; in our dynamic CBG, there is a hard constraint that the total budget is at most $B$. Second, there are no system dynamics in \cite{vu2019combinatorial}, i.e., unallocated troops in a round will not accumulate for the remaining rounds. Thus, the learner in \cite{vu2019combinatorial} has to allocate all the available $m$ troops in each round because unallocated troops do not contribute to future rewards. On the contrary, in our model, all rounds are coupled by the state dynamics that unallocated troops will automatically be rolled over to the next round. As a result, the learner must be more strategic and reserve some of its budget for future rounds.

Next, we focus on the action space of the dynamic CBG. 
Let {\color{black} $\tau \in [T+1]$} denote the stopping time of the algorithm {\color{black} defined earlier in Section \ref{sec:model}, which is the first round when either the time horizon is reached, or the consumption in that round exceeds learner's remaining budget for the first time.}
If we look back on the first $\tau-1$ rounds, the learner selects {\color{black} actions from the action set
\begin{equation} \tag{\ref{eq:action-set} revisited}
    \mc{S} = \left\{\bm{u} \in \mb{Z}_+^n: \sum_{i=1}^n u_i \leq m \right\}.
\end{equation}
In round $\tau$, either the time horizon is reached, or the learner could not draw an action from $\mcs$ due to insufficient remaining budget.}
In other words, during the dynamic game, the learner keeps selecting {\color{black} allocation actions} from $\mcs$ until it fails to do so in round $\tau$, {\color{black} and the game terminates immediately in round $\tau$.}
Let $S = \vert \mcs \vert$. One should note that the order of $S$ is exponential in terms of $m$ or $n$, and \(S = O\left(2^{\min \{m,n-1\}}\right)\) \cite{vu2019combinatorial}. Therefore, an efficient learning algorithm that uses the combinatorial structure of CBG is essential for obtaining a polynomial regret bound {\color{black} in terms of $m$ and $n$}.

To provide a path characterization for the allocation actions in $\mcs$, we use a similar idea as in \cite{vu2019combinatorial} to create a layered graph. Since the learner can allocate any integer-valued number of troops between $0$ and $m$ in each round {\color{black} before the dynamic game terminates, each layer between layer 1 and layer $n$} contains $m+1$ nodes, and any node in layer $n$ can be regarded as a destination node. {\color{black} An allocation action in $\mcs$} corresponds to a path from $s$ to any node in layer $n$. To unify all the destination nodes, we add an extra layer with one dummy node $d$ and connect all the nodes in layer $n$ to the dummy node $d$ using auxiliary edges. Therefore, the dummy node $d$ is the unique destination node such that any {\color{black} action in $\mcs$} corresponds to a path from $s$ to $d$. The revised layered graph consists of $n+2$ layers. The first and the last layers have one node, and each of all other layers consists of $m+1$ nodes. Thus, the total number of edges is $E = (m+1)[(m+2)(n-1)+4]/2 = O(nm^2)$. We denote the revised layered graph {\color{black} corresponding to $\mcs$} by $H_{m,n}$ (see Figure \ref{fig:revised-graph-sample-path} for an illustration).

The algorithm \textsc{Edge} works on the revised layered graph $H_{m,n}$. It should be noted that the auxiliary edges in $H_{m,n}$ do not represent battlefields or constitute the {\color{black} allocation actions of the CBG}. As a result, the weights of auxiliary edges do not change with time and remain equal to their initial value $1$ for all rounds. With a slight abuse of notation, we use $\bm{u} \in \{0,1\}^{E}$ to denote the characteristic vector of an $s,d$-path equivalent to the corresponding {\color{black} allocation action} of the learner in the CBG. Moreover, we use $\mcs \subseteq \{0,1\}^{E}$ to denote the set of all $s,d$-paths. Henceforth, we do not distinguish between {\color{black} the set of learner's allocation actions in the dynamic CBG} and the set of $s,d$-paths in the layered graph $H_{m,n}$. 

\begin{figure}[ht]
\centering
\includegraphics[width=.6\textwidth]{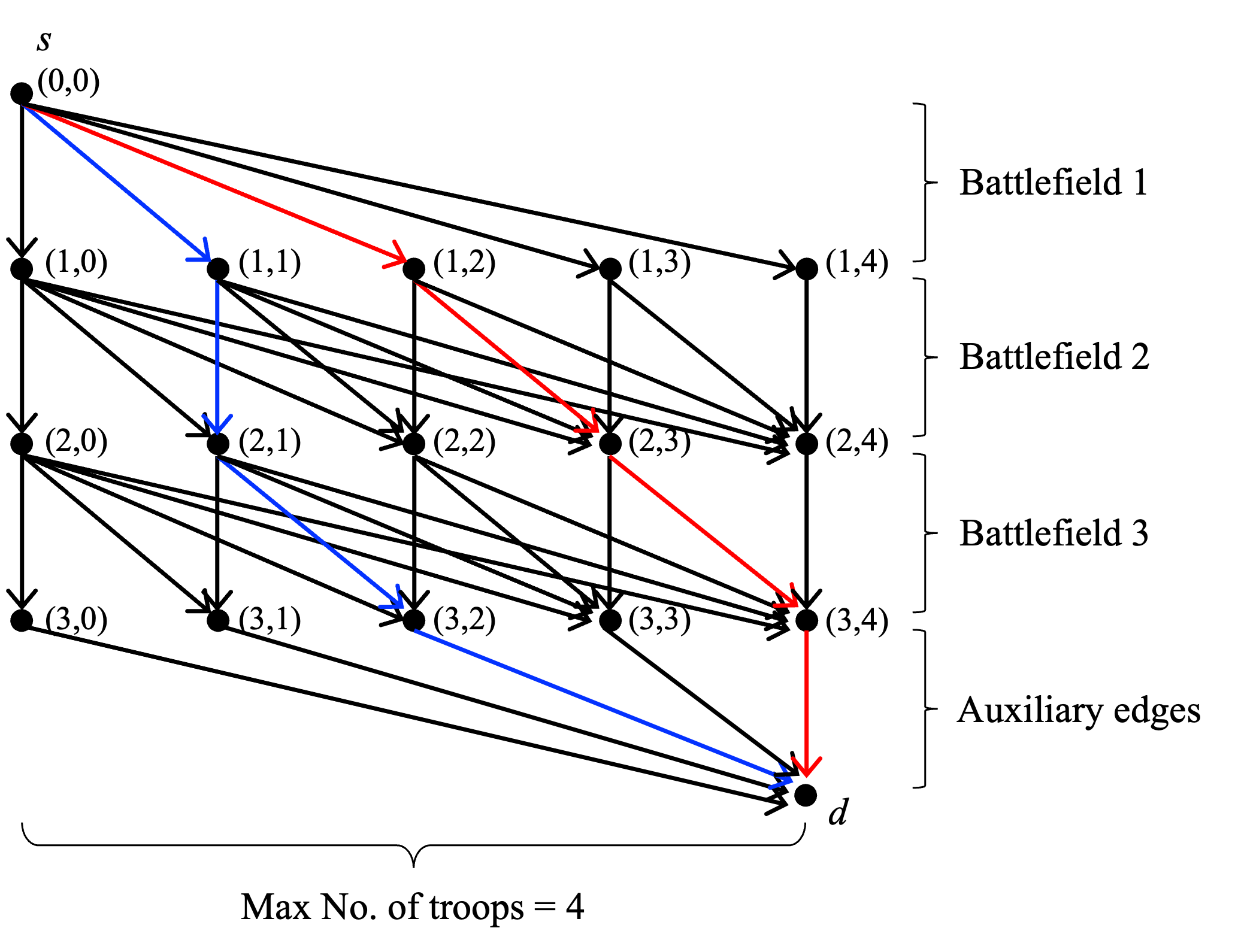}
\caption{An example of revised layered graph $H_{4,3}$ for CBG with $m=4$ and $n=3$. The blue path represents the allocation action $(1,0,1)$ with total consumption 2; the red path represents the allocation action $(2,1,1)$ with total consumption 4.}
\label{fig:revised-graph-sample-path}
\end{figure}

\subsection{Algorithm \textsc{Edge} with High-Probability Regret Bound} \label{subsec:high-prob-bound-algr}

{\color{black} Since \textsc{Edge} \cite{vu2019combinatorial} is an efficient algorithm designed for the repeated CBG without total budget constraint implemented on a layered graph, it is deployed by the learner for the primal algorithm ($\texttt{ALG}_1$) in Lagrangian game.}
In order to implement the algorithm \textsc{Edge} as a subroutine of \texttt{LagrangeBwK}, we need a bound on the \emph{realized regret} which holds with high probability.
{\color{black} In this subsection, the realized regret is defined as the difference between the optimal cumulative realized reward (which is the cumulative realized reward by following the best allocation action that maximizes the cumulative realized reward given the adversary's strategy) and the cumulative realized reward by following algorithm \textsc{Edge}. The expected regret is the expectation of the realized regret with respect to the algorithm's randomness.}
Note that {\color{black} bounding the realized regret} is a stronger requirement than bounding the expected regret. To that end, we extend the analysis in \cite{cesabianchi2012} and borrow some ideas from \cite{bartlett2008highprob} to bound the realized regret for the same algorithm. 
The high-probability bound {\color{black} for realized regret} is stated in the following theorem, the proof of which can be found in Appendix \ref{ap:proof-thm-high-prob-bound}.
Compared to the bound for expected regret stated in Lemma \ref{lemma:Edge-regret}, the bound for realized regret has an additional factor in \(T\) because of the large variance of the estimated reward. The result stated in Theorem \ref{thm:high-prob-Edge} satisfies the requirement in Lemma \ref{lemma:LagrangeBwK-regret}, and therefore the algorithm \textsc{Edge} can be implemented with algorithm \texttt{LagrangeBwK} as a subroutine.

\begin{theorem} \label{thm:high-prob-Edge}
Let \(\gamma = \frac{n}{\lambda^*}\sqrt{\frac{\ln S}{\left(\frac{n}{E\lambda^*}+1\right)ET^{2/3}}}\) and \(\eta = \frac{\gamma \lambda^*}{n}\) in algorithm \textsc{Edge}, {\color{black} where $S = \vert\mcs\vert$, and $\lambda^*$ is the smallest nonzero eigenvalue of the co-occurrence matrix \(M(\mu)\) of the exploration distribution $\mu$. Here, $\mu$ is a distribution over $\mcs$ and is an input into the algorithm, and $M(\mu) \triangleq \mbe_{\bm{u}\sim\mu}[\bm{u}\bm{u}^{\tp}]$.} With probability at least \(1 - O(\delta)\), the algorithm \textsc{Edge} guarantees that 
\begin{equation*}
    R_{\delta}(T) = O\left(T^{2/3}\sqrt{\left(\frac{n}{E\lambda^*}+1\right) E \ln(S/\delta)}\right). 
\end{equation*}
\end{theorem}


\subsection{Algorithm for Dynamic Colonel Blotto Game} \label{subsec:dynamic-cbg-algr}

In this section, we combine the algorithms \texttt{LagrangeBwK} and \textsc{Edge} to devise a hybrid algorithm with sublinear regret for the dynamic CBG. The hybrid algorithm for dynamic CBG with total budget constraint, named \texttt{LagrangeBwK-Edge}, is summarized in Algorithm \ref{alg:LagrangeBwK-Edge}. 
{\color{black} The learner follows the algorithm \texttt{LagrangeBwK-Edge} to play dynamic CBG with the adversary. As stated earlier in Section \ref{subsec:prelim-LagrangeBwK}, in order to play with the adversary strategically, the learner sets up a repeated Lagrangian game with fictitious primal and dual players and deploys two sub-algorithms for them to compete against each other. The learner coordinates the repeated Lagrangian game by feeding and getting feedback so that the Lagrangian game gradually converges to its Nash equilibrium, which is in turn the optimal mixed strategy for the learner against the adversary.}
Algorithm \textsc{Edge} is deployed as the primal algorithm ($\text{\texttt{ALG}}_{1}$) in \texttt{LagrangeBwK} {\color{black} to choose allocation action for CBG}. 
Algorithm Hedge, as introduced in \cite{freund1997hedge}, an algorithm for online learning in adversarial MAB (see Algorithm \ref{alg:hedge} in Appendix \ref{ap:algr} for the pseudocode), is deployed as the dual algorithm ($\texttt{ALG}_2$) in \texttt{LagrangeBwK} {\color{black} to choose resource}.

{\color{black}
In each round $t$, the primal algorithm \textsc{Edge} chooses an allocation action $\bm{u}_t \in \mcs$, and the dual algorithm Hedge chooses one among the two resources: troop (which consumes the budget $B$) and time (which consumes the time horizon $T$). The learner follows the primal algorithm and plays $\bm{u}_t$ against the adversary, consuming $w_t(\bm{u}_t)$ unit of troops and receiving $r_t(\bm{u}_t)$ unit of reward. Then, it computes the realized} payoffs of the Lagrangian game as follows: 
\begin{align}
    \mc{L}_t^{\text{troop}} & = r_t(\bm{u}_t) + 1 - \frac{T}{B} \cdot w_t(\bm{u}_t), \label{eq:revised-Lagrange-function-troop} \\ 
    \mc{L}_t^{\text{time}} & = r_t(\bm{u}_t). \label{eq:revised-Lagrange-function-time}  
\end{align}
{\color{black}
Depending on whether the dual algorithm Hedge chooses troop or time, the learner feeds $\mc{L}_t^{\text{troop}}$ or $\mc{L}_t^{\text{time}}$ to algorithm \textsc{Edge} as reward for weight updating. Then, the learner feeds both $\mc{L}_t^{\text{troop}}$ and $\mc{L}_t^{\text{time}}$ to algorithm Hedge as cost for weight updating.}

One interpretation of Equations \eqref{eq:revised-Lagrange-function-troop} and \eqref{eq:revised-Lagrange-function-time} is that the payoff of Lagrangian game is a trade-off between reward $r_t(\bm{u}_t)$ and consumption $w_t(\bm{u}_t)$. If the reward outweighs the consumption, {\color{black} consumption of the budget is not a serious issue, and} the dual algorithm is more likely to choose time. {\color{black} When the dual algorithm chooses time,} the true reward of one-shot CBG $r_t(\bm{u}_t)$ is returned to the primal algorithm {\color{black} as if the budget constraint did not exist.}
However, if the consumption outweighs the reward, {\color{black} an alert of overconsumption is raised, and} the dual algorithm is more likely to choose troop. {\color{black} When the dual algorithm chooses troop, it warns the primal algorithm by returning} the adjusted value $\mc{L}_t^{\text{troop}} \leq r_t(\bm{u}_t)$ to the primal algorithm so that the budget will not be exhausted too quickly. 

With the result of Theorem \ref{thm:high-prob-Edge}, we obtain the following main theorem, the proof of which can be found in Appendix \ref{ap:proof-thm-LagrangeBwK-Edge}.

\begin{theorem} \label{thm:LagrangeBwK-Edge}
Let $\texttt{ALG}_1$ be algorithm \textsc{Edge} with \(\gamma = \frac{n}{\lambda^*}\sqrt{\frac{\ln S}{\left(\frac{n}{E\lambda^*}+1\right)ET^{2/3}}}\), \(\eta = \frac{\gamma \lambda^*}{(1+c)n}\),
and high-probability regret bound \(R_{1,\delta/T}(T) = O\left(T^{2/3}\sqrt{\left(\frac{n}{E\lambda^*}+1\right) E \ln(ST/\delta)}\right)\) {\color{black} following the same notations as in Theorem \ref{thm:high-prob-Edge}}. Moreover, let $\texttt{ALG}_2$ be algorithm Hedge with high-probability regret bound \(R_{2,\delta/T}(T) = O\left(\sqrt{T \ln (T/\delta)}\right)\). Then, with probability at least \(1 - O(\delta T)\), the regret of Algorithm \ref{alg:LagrangeBwK-Edge} for dynamic CBG is at most  
\begin{align}\nonumber
    R(T)  = O \left(T^{1/6}\sqrt{\frac{nB}{\lambda^*} \ln \left(T/\delta\right)}\right). 
\end{align}
Moreover, the time complexity of \texttt{LagrangeBwK-Edge} is $O(n^2 m^4 T)$. 
\end{theorem}

\begin{algorithm}[t]
\caption{\texttt{LagrangeBwK-Edge} Algorithm} \label{alg:LagrangeBwK-Edge}
\textbf{Input:} $B$, $T$, $m=cB/T$, $n$, primal algorithm \textsc{Edge}, dual algorithm Hedge. 
\begin{algorithmic}[1]
    \For{round $t = 1, \cdots, T$}
    \State The adversary plays.
    \State \multiline{Sample a path $\bm{u}_t \in \mathcal{S}$ by edges' weights using algorithm \textsc{Edge}.}
    \If{unsuccessful due to insufficient budget}
        \State Terminate the algorithm.
    \EndIf 
    \State \multiline{Observe the reward $r_t(\bm{u}_t)$, and compute $\mc{L}_t^{\text{troop}}$ and $\mc{L}_t^{\text{time}}$.}
    \State Choose either troop or time using algorithm Hedge. 
    \If{algorithm Hedge chooses troop}
        \State Pass $\mc{L}_t^{\text{troop}}$ to algorithm \textsc{Edge} as reward for weight updating. 
    \Else
        \State Pass $\mc{L}_t^{\text{time}}$ to algorithm \textsc{Edge} as reward for weight updating.
    \EndIf
    \State Update the weights of non-auxiliary edges using \textsc{Edge}.
    \State \multiline{Pass $\mc{L}_t^{\text{troop}}$ and $\mc{L}_t^{\text{time}}$ to algorithm Hedge as cost and update their weights.}
    \EndFor
   \end{algorithmic}
\end{algorithm}


\section{Simulation Results} \label{sec:simulation}

In this section, we show the simulation results to demonstrate the performance of the algorithm \texttt{LagrangeBwK-Edge}. 
Firstly, we consider two types of adversaries: i) the \emph{static adversary}, who at each round $t$ deterministically assigns one troop to each of the first two battlefields, and ii) the \emph{random adversary}, who at each round $t$ repeatedly draws a battlefield by uniform distribution and assigns one troop to it until {\color{black} two troops are assigned in that round.}
{\color{black} For the learner,} we fix $B/T = 2$ and $c = 2$. That is, the per-round average budget of the learner is 2, and the learner is willing to distribute at most {\color{black} 4 troops} to each round.
Moreover, all battlefields are assumed to have the same weights. 
{\color{black} We compare the performance of \texttt{LagrangeBwK-Edge} with uniformly random strategy, i.e., the learner uniformly randomly samples an allocation action in $\mcs$ (defined in Eq. \eqref{eq:action-set}) in each round until the stopping time. For each setting of parameters, the algorithms are executed 20 times, and the average results of the 20 trials are taken.}
The regret results for static and random adversaries are shown in Figs. \ref{subfig:static} and \ref{subfig:unif_rand}.

Next, we consider a \emph{super adversary} who repeatedly assigns 2 troops to each of the fixed $(n-1)$ battlefields in each round, leaving one battlefield unoccupied. The learner, however, does not know {\color{black} which battlefield is left unoccupied by the adversary. In this scenario, for the learner,} we set $B/T = 1$ and $m = 4$. In this case, the learner's per-round average budget is 1, and it is not willing to distribute more than 4 {\color{black} troops} to each round. Hence, the adversary is much more powerful than the learner as it possesses $2(n-1)$ times more troops in total than the learner. The learner's optimal strategy is to allocate 1 troop to the battlefield unoccupied by the adversary in each round. Therefore, the learner has to quickly identify which battlefield is left unoccupied without consuming the budget excessively in the early stage. {\color{black} We also compare the performance of \texttt{LagrangeBwK-Edge} with uniformly random strategy, and each algorithm is executed 20 times with the average results taken.} The numerical results for the super adversary are shown in Fig. \ref{subfig:super}. 

The numerical results in all cases {\color{black} are two fold. Firstly, the simulation results demonstrate a regret sublinear in $T$ for algorithm \texttt{LagrangeBwK-Edge}}, which justifies our theoretical bound in Theorem \ref{thm:LagrangeBwK-Edge} {\color{black} and proves the effectiveness of learning. By contrast, the regret by following uniformly random strategy is essentially linear in $T$. Secondly, compared to uniformly random strategy, our algorithm achieves significantly better regret.}


\begin{figure}[ht]
    \captionsetup{font=small}
    \centering
    \begin{subfigure}[h]{0.475\textwidth}
    \centering
    \includegraphics[trim={0 0.2in 0.3in 0.25in},clip,width=\textwidth]{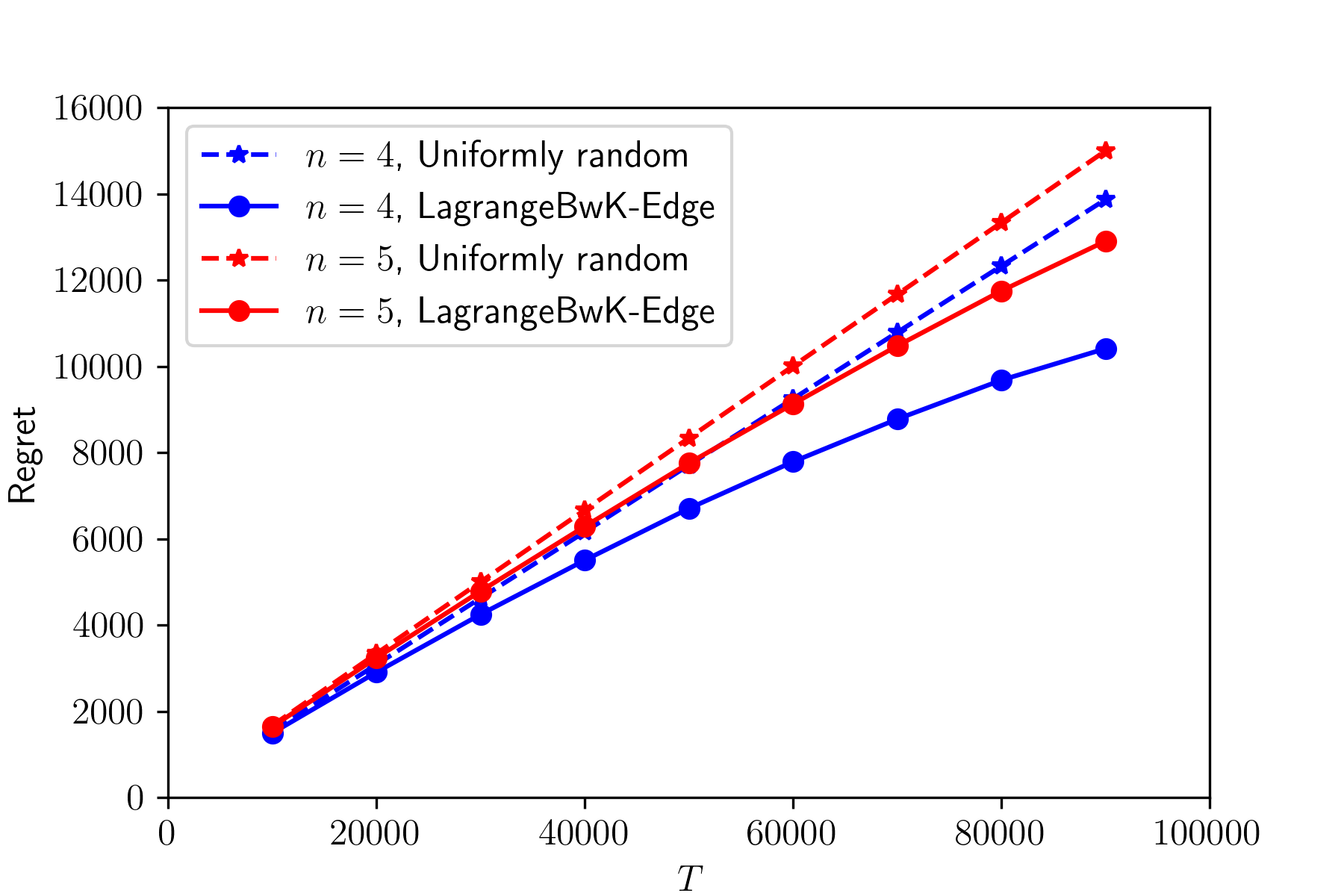}
    \subcaption{Static adversary.}
    \label{subfig:static}
    \end{subfigure}
    \begin{subfigure}[h]{0.475\textwidth}
    \centering
    \includegraphics[trim={0 0.2in 0.3in 0.25in},clip,width=\textwidth]{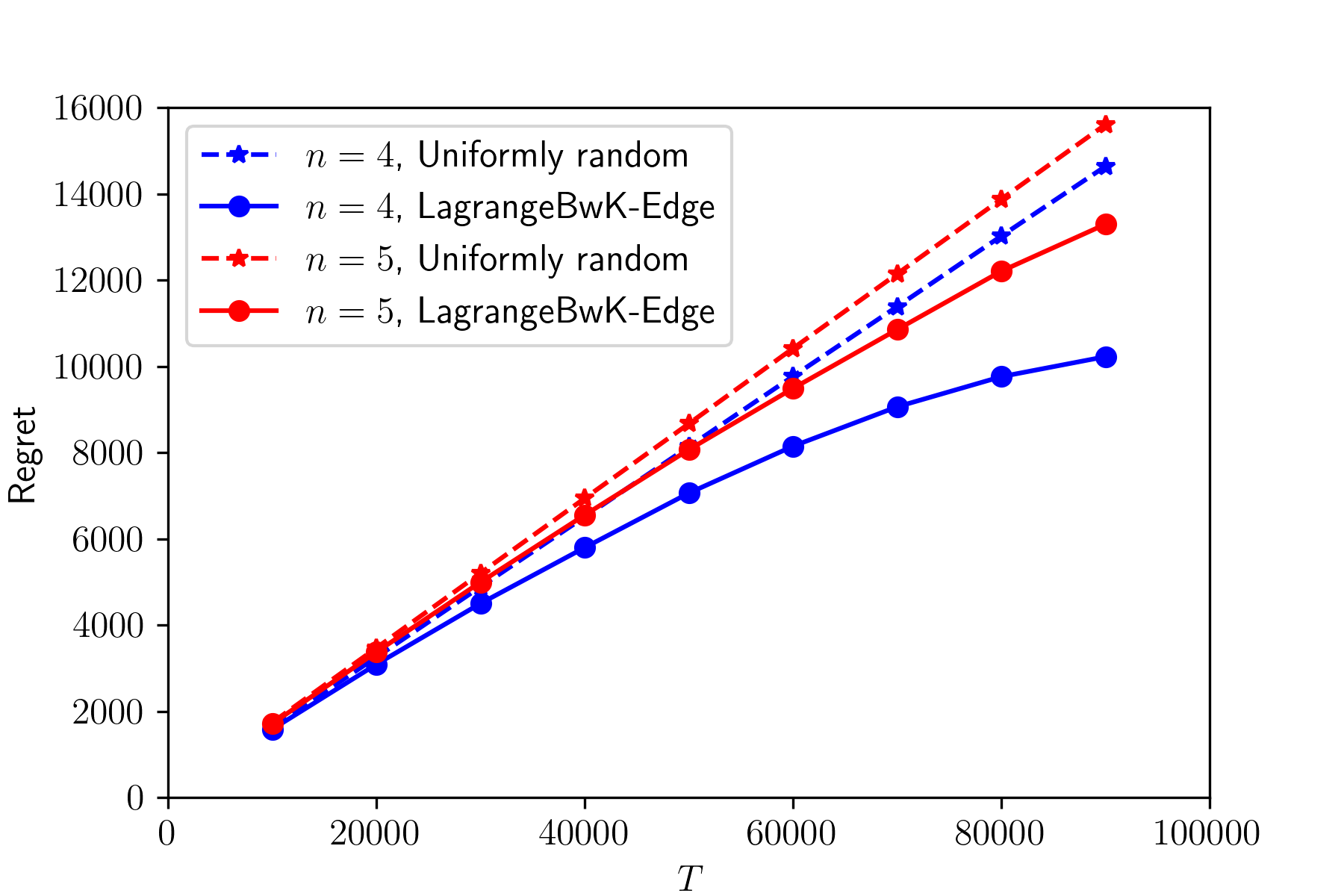}
    \subcaption{Random adversary.}
    \label{subfig:unif_rand}
    \end{subfigure}
    \begin{subfigure}[h]{0.475\textwidth}
    \centering
    \includegraphics[trim={0 0.2in 0.3in 0.25in},clip,width=\textwidth]{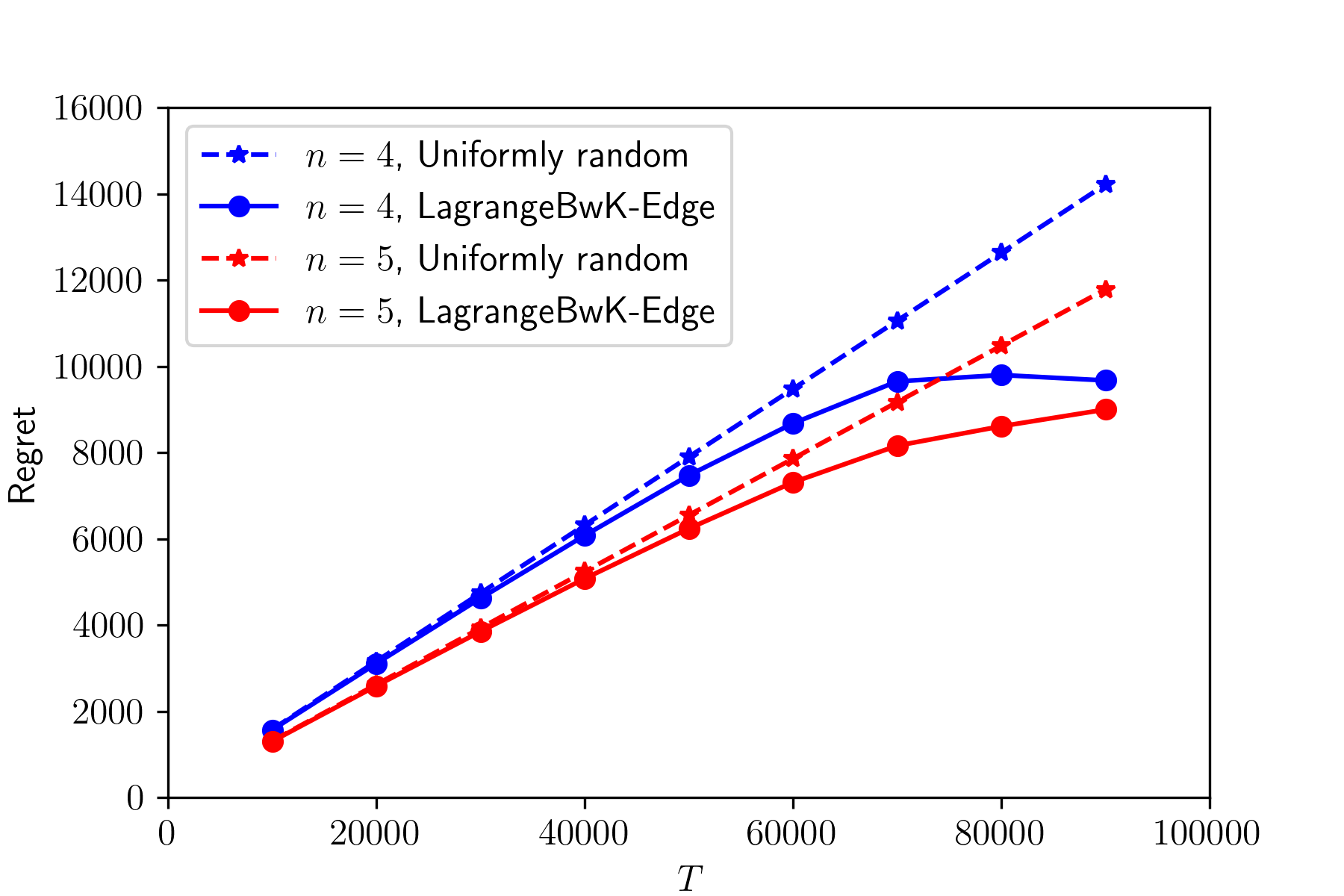}
    \subcaption{Super adversary.}
    \label{subfig:super}
    \end{subfigure}
    \caption{Simulation results for \texttt{Lagrange-BwK}.}
    \label{fig:simulation}
\end{figure}


\section{Conclusion} \label{sec:conclusion}
In this paper, we consider a dynamic Colonel Blotto game with one learner and one adversary. The learner has a limited budget over a finite time horizon. The learner faces {\color{black} a two-level decision problem}: on an upper level, how many troops to allocate to each stage, and on a lower level, how to distribute the troops among battlefields in each stage. On the upper level, the budget-constrained dynamic CBG is modeled by BwK. On the lower level, the single-shot CBG at each stage has such a combinatorial structure that lies in the framework of CB. {\color{black} We have combined the bandits models and devised a polynomial-time learning algorithm \texttt{LagrangeBwK-Edge} by transforming the CBG to a path planning problem on a directed graph and deploying a CB algorithm \textsc{Edge} for the path planning problem as a subroutine of a BwK algorithm \texttt{LagrangeBwK}. To enable this, we have extended the analysis of algorithm \textsc{Edge}. Our algorithm \texttt{LagrangeBwK-Edge} admits a regret bound sublinear in time horizon and polynomial in other parameters that holds with high probability.}


As a future direction, one can consider a version of the model that we consider here in which an adversary is replaced by another strategic player. Also, it would be interesting to improve our regret bounds using a more detailed UCB analysis.



\bibliographystyle{alpha}
\bibliography{references}

\newpage

\begin{appendices}

\section{Existing Algorithms}
\label{ap:algr}

For the sake of completeness, in this appendix, we provide a detailed description of the existing algorithms (Algorithms \ref{alg:wp}--\ref{alg:hedge}) that we use as subroutines in our main algorithm.

\begin{algorithm}[h]
\caption{Weight-Pushing Algorithm, \cite[Algorithm 2]{vu2019combinatorial_arxiv}}
\label{alg:wp} 
\textbf{Input:} $G_{m,n}=(\mathcal{V},\mathcal{E})$, $t\in [T]$, $w_e^t, \forall e \in \mathcal{E}$.
\begin{algorithmic}[1]
    \State \textbf{Initialization:} node set $Q = \{s\}$, $u_0 = s$, $k=0$.
    \For{$k \leq n$}
        \State \multiline{Sample a node $u_{k+1}$ from the set of children of $u_k$ with probability
        \begin{equation*}
            \frac{w_{e_{[u_k,u_{k+1}]}}^t H^t(u_{k+1},d)}{H^t(u_k,d)}
        \end{equation*}
        where $H^t(u,v) \triangleq \sum_{\bm{p} \in \mathcal{P}_{(u,v)}} \prod_{e\in \bm{p}} w_e^t$ and $\mathcal{P}_{(u,v)}$ denotes the set of all paths from $u$ to $v$.}
        \State $Q = Q \cup \{u_{k+1}\}$.
    \EndFor
\end{algorithmic}
\textbf{Output:} $\bm{p}^t \in \mathcal{P}_{s,d}$ which passes the nodes in $Q$. 
\end{algorithm}

\begin{algorithm}[h]
\caption{Co-occurrence Matrix Computation Algorithm, \cite[Algorithm 3]{vu2019combinatorial_arxiv}}
\label{alg:matcomp} 
\textbf{Input:} $G_{m,n}=(\mathcal{V},\mathcal{E})$, $\tilde{w}_e, \forall e \in \mathcal{E}$.
\begin{algorithmic}[1]
    \State \multiline{Compute $H(u,v) \triangleq \sum_{\bm{p} \in \mathcal{P}_{(u,v)}} \prod_{e\in \bm{p}} \tilde{w}_e$ for all $u,v \in \mathcal{V}$ using dynamic programming.}
    \For{$e_1=e_{[u_1, v_1]}\in \mathcal{E}$}
    \State Compute $M(\mu_{\tilde{\bm{w}}})_{e_1,e_1} = \frac{H(s,u_1)\tilde{w}_{e_1}H(v_1,d)}{H(s,d)}$.
    \For{$e_2=e_{[u_2,v_2]} \in \mathcal{E}, \text{idx}(e_2) > \text{idx}(e_1)$}
    \State \multiline{Compute $M(\mu_{\tilde{\bm{w}}})_{e_1,e_2} = \frac{H(s,u_1)\tilde{w}_{e_1}H(v_1,u_2)\tilde{w}_{e_2}H(v_2,d)}{H(s,d)}$ where $\text{idx}(e)$ returns the index of $e$ in $\mathcal{E}$.}
    \EndFor
    \EndFor
    \For{$e_1, e_2 \in \mathcal{E}, \text{idx}(e_2) < \text{idx}(e_1)$}
    \State $M(\mu_{\tilde{\bm{w}}})_{e_1,e_2} = M(\mu_{\tilde{\bm{w}}})_{e_2,e_1}$.
    \EndFor
\end{algorithmic}
\textbf{Output:} Co-occurrence matrix $M(\mu_{\tilde{\bm{w}}})$.
\end{algorithm}

\begin{algorithm}[h]
\caption{Hedge, \cite[Fig. 1]{freund1997hedge}}
\label{alg:hedge} 
\textbf{Input:} $\beta \in (0,1)$, $T$.
\begin{algorithmic}[1]
    \State \multiline{\textbf{Initialization:} $w_a^t = 1$ for each arm $a \in \mathcal{A}$ where $\mathcal{A}$ denotes the set of all arms.}
    \For{round $t = 1, \cdots, T$}
    \State Sample an arm $a_t \in \mathcal{A}$ with probability \(p^t_{a_t} = \frac{w^t_{a_t}}{\sum_{a \in \mathcal{A}} w^t_a}\).
    \State Receive the cost $\{c^t_a\}_{a \in \mathcal{A}}$ for all arms $a \in \mathcal{A}$.
    \State Update the weight $w^{t+1}_a = w^t_a \cdot \beta^{c^t_a}$ for all $a \in \mathcal{A}$.
    \EndFor
\end{algorithmic}
\end{algorithm}

\newpage
\section{Concentration Inequalities} \label{ap:concentration-ineq}

\begin{lemma} [Bernstein's inequality for martingales, Lemma A.8 in \cite{cesabianchi2006book}] \label{lemma:bernstein} 
Let \(Y_1, Y_2, \cdots\) be a \emph{martingale difference sequence} (i.e., \(\mbe[Y_t\vert Y_{t-1},\cdots,Y_1]=0 \: \forall t \in \mbz_+\)). Suppose that \(\lvert Y_t \rvert \leq c\) and \(\mbe[Y_t^2 \vert Y_{t-1}, \cdots, Y_1] \leq v\) almost surely for all \(t \in \mbz_+\). For any \(\delta> 0\), 
\begin{equation*}
    \text{\textup{P}}\left(\sum_{t=1}^T Y_t > \sqrt{2 Tv \ln(1/\delta)} + \frac{2}{3} c\ln(1/\delta) \right) \leq \delta. 
\end{equation*}
\end{lemma}

\begin{lemma} [Azuma-Hoeffding's inequality, Lemma A.7 in \cite{cesabianchi2006book}] \label{lemma:azuma-hoeffding}
Let \(Y_1, Y_2, \cdots\) be a martingale difference sequence and \(\vert Y_t\vert \leq c\) almost surely for all \(t \in \mbz_+\). For any \(\delta > 0\), 
\begin{equation*}
    \text{\textup{P}}\left(\sum_{t=1}^T Y_t > \sqrt{2Tc^2 \ln(1/\delta)}\right) \leq \delta. 
\end{equation*}
\end{lemma}


\section{Proof of Theorem \ref{thm:high-prob-Edge}} \label{ap:proof-thm-high-prob-bound}

The proof of Theorem \ref{thm:high-prob-Edge} is built upon several results from combinatorial bandits \cite{cesabianchi2012,bartlett2008highprob,dani2008linband}. In the following, we first state some useful lemmas from  \cite{cesabianchi2012,vu2019combinatorial}, and then use them to prove Theorem \ref{thm:high-prob-Edge}.

Throughout the appendix, let $\bm{r}_t(\bm{u}) = \bm{l}_t^{\tp} \bm{u}$ be the reward by playing arm $\bm{u}$ in round $t$, and $\hat{r}_t(\bm{u}) = \hat{\bm{l}}_t^{\tp} \bm{u}$ be the estimated reward where $\hat{\bm{l}}_t$ is the unbiased cost estimator computed in Algorithm \ref{alg:Edge}.

\begin{lemma} \label{lemma:comband-results}
Let $\lambda^*$ be the smallest nonzero eigenvalue of the co-occurrence matrix $M(\mu)=\mb{E}_{\bm{u}\sim\mu}[\bm{u}\bm{u}^{\tp}]$ for the exploration distribution $\mu$ in algorithm \textsc{Edge}. For all $\bm{u} \in \mathcal{S} \subseteq \{0,1\}^E$ and all $t\in[T]$, the following relations hold: 
\begin{itemize}
    \item \(\vert\vert\bm{u}\vert\vert^2 = \sum_{i=1}^{E} u_i^2 = n\)
    \item \(\vert\hat{r}_t(\bm{u})\vert = \vert\hat{\bm{l}}_t^{\tp} \bm{u}\vert \leq \frac{n}{\gamma \lambda^*}\)
    \item \(\bm{u}^{\tp} C_t^{-1} \bm{u} \leq \frac{n}{\gamma \lambda^*}\)
    \item \(\sum_{\bm{u} \in \mc{S}} p_t(\bm{u}) \bm{u}^{\tp} C_t^{-1} \bm{u} \leq E\)
    \item \(\mb{E}_t[(\hat{\bm{l}}_t^{\tp} \bm{u})^2] \leq \bm{u}^{\tp} C_t^{-1} \bm{u}\)
\end{itemize}
where in the last expression \(\mb{E}_t[\cdot]:= \mb{E}[\cdot\vert\bm{u}_{t-1},\cdots,\bm{u}_1]\).
\end{lemma}
  
\begin{lemma} \cite[Appendix A]{cesabianchi2012} \label{lemma:cesabianchi-baseline}
By choosing \(\eta = \frac{\gamma \lambda^*}{n}\) such that \(\eta \vert\hat{r}_t(\bm{u})\vert \leq 1\), for all \(\bm{u}^* \in \mc{S}\), 
\begin{align}\label{eq:proof-lemma-baseline}
    \sum_{t=1}^T \hat{r}_t(\bm{u}^*) - \frac{1}{\eta}\ln S 
    & \leq \frac{1}{1-\gamma} \sum_{t=1}^T \sum_{\bm{u}\in\mc{S}} p_t(\bm{u}) \hat{r}_t(\bm{u}) \cr
    & + \frac{\eta}{1-\gamma} \sum_{t=1}^T \sum_{\bm{u}\in\mc{S}} p_t(\bm{u})\hat{r}_t(\bm{u})^2\cr
    &  - \frac{\gamma}{1-\gamma} \sum_{t=1}^T \sum_{\bm{u}\in\mc{S}} \hat{r}_t(\bm{u}) \mu(\bm{u}).  
\end{align}
\end{lemma}
\begin{proof}
This is a straightforward result of Equations (A.1)-(A.3) in \cite{cesabianchi2012} by flipping the sign of $\eta$.
\end{proof}

Lemma \ref{lemma:cesabianchi-baseline} provides a baseline for bounding the regret. We will proceed to bound each summation in \eqref{eq:proof-lemma-baseline}. The following lemma, derived from Bernstein's inequality (Lemma \ref{lemma:bernstein}), provides a high-probability bound on the left side of \eqref{eq:proof-lemma-baseline}. 

\begin{lemma} \label{lemma:bernstein-estimated-reward}
With probability at least \(1 - \delta\), for all \(\bm{u} \in \mc{S}\), it holds that 
\begin{align*}
    \sum_{t=1}^T r_t(\bm{u}) - \sum_{t=1}^T \hat{r}_t(\bm{u}) 
     \leq \sqrt{2T \left(\frac{n}{\gamma \lambda^*}\right) \ln (S/\delta)} + \frac{2}{3}\left(\frac{n}{\gamma \lambda^*}+1\right)\ln (S/\delta).
\end{align*}
\end{lemma}
\begin{proof}
Fix \(\bm{u} \in \mc{S}\). Define \(Y_t = r_t(\bm{u}) - \hat{r}_t(\bm{u})\). Then \(\{Y_t\}_{t=1}^T\) is a martingale difference sequence. From Lemma \ref{lemma:comband-results}, we know that \(\vert Y_t\vert \leq \frac{n}{\gamma \lambda^*}+1\). Let \(\mb{E}_t[Y_t^2] = \mb{E}[Y_t^2\vert Y_{t-1}, \cdots, Y_1]\). Then,  
\begin{align*}
    \mb{E}_t[Y_t^2] \leq \mb{E}_t[(\hat{\bm{l}}_t^{\tp} \bm{u})^2] \leq \bm{u}^{\tp} C_t^{-1} \bm{u} \leq \frac{n}{\gamma \lambda^*}.
\end{align*}
Using Bernstein's inequality, with probability at least \(1-\delta/S\), 
\begin{equation*}
    \sum_{t=1}^T Y_t \leq \sqrt{2T \left(\frac{n}{\gamma \lambda^*}\right) \ln (S/\delta)} + \frac{2}{3}\left(\frac{n}{\gamma \lambda^*}+1\right)\ln (S/\delta)
\end{equation*}
The lemma now follows by using the above inequality and taking a union bound over all \(\bm{u} \in \mc{S}\). 
\end{proof}

The following two lemmas obtained in \cite{bartlett2008highprob} provide a high-probability bound on the first and second summands on the right side of \eqref{eq:proof-lemma-baseline}. The proofs of these lemmas, which are omitted here due to space limitation, use a direct application of Bernstein's inequality and Azuma-Hoeffding's inequalitiy.

\begin{lemma}\cite[Lemma 6]{bartlett2008highprob} \label{lemma:estimated-reward-bound}
With probability at least \(1-\delta\), 
\begin{align*}
     \sum_{t=1}^T \sum_{\bm{u}\in\mc{S}} p_t(\bm{u}) \hat{r}_t(\bm{u}) - \sum_{t=1}^T r_t(\bm{u}_t)& \leq \left(\sqrt{E} + 1\right) \sqrt{2T \ln(1/\delta)} \cr 
     &+ \frac{4}{3}\ln(1/\delta)\left(\frac{n}{\gamma \lambda^*}+1\right). 
\end{align*}
\end{lemma}

\begin{lemma} \cite[Lemma 8]{bartlett2008highprob} \label{lemma:square-reward-bound} 
With probability at least \(1 - \delta\), 
\begin{equation*}
    \sum_{t=1}^T \sum_{\bm{u}\in\mc{S}} p_t(\bm{u}) \hat{r}_t(\bm{u})^2 
    \leq ET + \frac{n}{\gamma \lambda^*}\sqrt{2T \ln(1/\delta)}. 
\end{equation*}
\end{lemma}

\smallskip
Now we are ready to complete the proof of Theorem \ref{thm:high-prob-Edge}. Using Lemma \ref{lemma:bernstein-estimated-reward} and because \(\sum_{t=1}^T r_t(\bm{u}) \geq 0\) for every \(\bm{u}\in\mc{S}\), we can bound the last term on the right side of \eqref{eq:proof-lemma-baseline} with probability at least \(1-\delta\) as
\begin{align*}
    - \gamma \sum_{t=1}^T \sum_{\bm{u}\in\mc{S}} \hat{r}_t(\bm{u}) \mu(\bm{u}) 
    & \leq \gamma \sqrt{2T\left(\frac{n}{\gamma \lambda^*}\right)\ln (S/\delta)} \\
    & \quad + \frac{2}{3}\gamma \left(\frac{n}{\gamma \lambda^*}+1\right) \ln (S/\delta)
\end{align*}
Using Lemmas \ref{lemma:bernstein-estimated-reward} to \ref{lemma:square-reward-bound} and the above inequality in Lemma \ref{lemma:cesabianchi-baseline}, with probability at least \(1 - 4\delta\), for all \(\bm{u} \in \mc{S}\) we have, 
\begin{align*}
    \sum_{t=1}^T r_t(\bm{u}) - \sum_{t=1}^T r_t(\bm{u}_t) 
    & \leq \sqrt{2T \left(\frac{n}{\gamma \lambda^*}\right) \ln (S/\delta)} \\
    & \quad + 3\left(\frac{n}{\gamma \lambda^*}+1\right)\ln (S/\delta) \\
    & \quad + (\sqrt{E} + 1) \sqrt{2T \ln (1/\delta)} \\
    & \quad + \frac{\gamma \lambda^*}{n} ET + \sqrt{2T \ln (S/\delta)} + \gamma T. 
\end{align*}
Finally, if we set
\begin{equation*}
    \gamma = \frac{n}{\lambda^*}\sqrt{\frac{\ln S}{\left(\frac{n}{E\lambda^*}+1\right)ET^{2/3}}}, 
\end{equation*}
the following regret bound can be obtained:
\begin{align*}
     \sum_{t=1}^T r_t(\bm{u}) &- \sum_{t=1}^T r_t(\bm{u}_t) \\
    & \leq \sqrt{2 \left(\frac{n}{E\lambda^*}+1\right)^{1/2} T^{4/3} E^{1/2} (\ln (S/\delta))^{1/2}} \\
    & + 3 \sqrt{\left(\frac{n}{E\lambda^*}+1\right) E T^{2/3} \ln (S/\delta)} + 3 \ln (S/\delta) \\
    &  + \sqrt{\left(\frac{n}{E\lambda^*}+1\right) E T^{4/3} \ln S} \\
    &  + (\sqrt{E}+1)\sqrt{2T \ln(1/\delta)} + \sqrt{2T \ln (S/\delta)} \\
    &  = O\left(T^{2/3}\sqrt{\left(\frac{n}{E\lambda^*}+1\right) E \ln(S/\delta)}\right). 
\end{align*}

\section{Proof of Theorem \ref{thm:LagrangeBwK-Edge}} \label{ap:proof-thm-LagrangeBwK-Edge}

It has been shown in \cite{freund1997hedge,immorlica2019lagrangebwk} that the algorithm Hedge achieves the high-probability regret bound of
\begin{equation*}
    R_{\delta}(T) = O\left(\sqrt{T \ln (\vert A\vert/\delta)}\right),
\end{equation*}
where \(\vert A\vert\) denotes the cardinality of action set, which in our setting is the number of resources. Since we only have two types of resources, namely the time and the troops, we have \(\vert A\vert=2=O(1)\). 

From Equation \eqref{eq:revised-Lagrange-function-troop}, we have \(\vert\mc{L}_t^{\text{troop}}\vert \leq \max\{2, \vert 1-c \vert \} \leq 1+c\) for $c \geq 1$. As a result, the actual reward \(r_t(\bm{u})\), estimated reward \(\hat{r}_t(\bm{u})\), and hence the regret bound of Theorem \ref{thm:high-prob-Edge} are all scaled up by at most a constant factor \(1+c\). Hence, in order to satisfy the assumption of Theorem \ref{thm:high-prob-Edge} given in \cite{cesabianchi2012} that \(\eta \vert\hat{r}_t(\bm{u})\vert \leq 1\), we set 
\begin{equation*}
    \eta = \frac{\gamma \lambda^*}{(1+c)n} = \frac{1}{1+c} \sqrt{\frac{\ln S}{\left(\frac{n}{E\lambda^*}+1\right)ET^{2/3}}}.
\end{equation*}
On the other hand, from Theorem \ref{thm:high-prob-Edge}, the high-probability regret bound of algorithm \texttt{Lagrange-Edge} is at most 
\begin{equation*}
    R_{\delta}(T) = O\left(T^{2/3}\sqrt{\left(\frac{n}{E\lambda^*}+1\right) E \ln(S/\delta)}\right).
\end{equation*}

Now, using Lemma \ref{lemma:LagrangeBwK-regret}, and noting that \(B_0\) is scaled to \(\frac{B}{(cB/T)}=\frac{T}{c}\) such that \(O\left(\frac{T}{B_0}\right)=O(1)\), we obtain that with probability at least \(1-O(\delta T)\), it holds that 
\begin{align*}
    R(T) & \leq O(1) \cdot \Bigg( O\left(T^{2/3}\sqrt{\left(\frac{n}{E\lambda^*}+1\right) E \ln(ST/\delta)}\right) \\
    & \qquad \qquad + O\left(\sqrt{T \ln (T/\delta)}\right) \Bigg) \\
    & = O\left(T^{2/3}\sqrt{\left(\frac{n}{E\lambda^*}+1\right) E \ln(ST/\delta)}\right). 
\end{align*}
Finally, we recall that \(E = O(nm^2)\), \(m = O(B/T)\), and \(S = O\left(2^{\min \{n-1,m\}}\right)\), such that \(\ln S \leq O(m) = O(B/T)\). Substituting these relations into the above inequality we get 
\begin{align*}
    R(T) &\leq O\left(T^{2/3}\sqrt{\left(\frac{T^2}{B^2 \lambda^*}+1\right) n \left(\frac{B}{T}\right)^3 \ln(T/\delta)}\right)\cr
    &= O \left(T^{1/6}\sqrt{\frac{nB}{\lambda^*} \ln \left(T/\delta\right)}\right).
\end{align*}


\end{appendices}

\end{document}